\documentclass[10pt]{article} 

\usepackage[accepted]{rlj}           

\usepackage{amssymb}            
\usepackage{mathtools}          
\usepackage{mathrsfs}           
\usepackage{graphicx}           
\usepackage{subcaption}         
\usepackage[space]{grffile}     
\usepackage{url}                
\usepackage{lipsum}             

\definecolor{brightmaroon}{rgb}{0.76, 0.13, 0.28}
\usepackage[capitalize,noabbrev]{cleveref}

\usepackage{titlesec}
\usepackage{algorithm}
\usepackage{algorithmic}
\usepackage{amsmath}
\usepackage{amssymb}
\usepackage{mathtools}
\usepackage{amsthm}
\usepackage{booktabs}

\newtheorem{theorem}{Theorem}[section]

\theoremstyle{definition}

\newtheorem{Proposition}{Proposition}
\newtheorem{Assumption}{Assumption}
\newtheorem{innercustomthm}{Proposition}
\newenvironment{custom}[1]
  {\renewcommand\theinnercustomthm{#1}\innercustomthm}
  {\endinnercustomthm}
\crefname{Assumption}{Assumption}{Assumptions}
\theoremstyle{remark}

\DeclareMathOperator{\E}{\mathbb{E}}

\newcommand{\abs}[1]{\left\lvert#1\right\rvert}
\newcommand{\norm}[1]{\left\lVert#1\right\rVert}
\newcommand{\ep}{\epsilon}

\newcommand{\cY}{\mathcal{Y}}

\newcommand{\cH}{\mathcal{H}}

\DeclareMathOperator*{\argmax}{argmax}

\title{An  Optimisation Framework for Unsupervised Environment Design}

\setrunningtitle{An Optimisation Framework for Unsupervised Environment Design}

\author{Nathan Monette\textsuperscript{1, $\dagger$, $*$}, 
Alistair Letcher\textsuperscript{2}, 
Michael Beukman\textsuperscript{2}, 
Matthew T. Jackson\textsuperscript{2}, Alexander Rutherford\textsuperscript{2}, Alexander D. Goldie\textsuperscript{2}, Jakob N. Foerster\textsuperscript{2}}

\emails{$^\dagger$nmonette@uci.edu}

\affiliations{
$^{1}$\textbf{University of California Irvine}\\
$^{2}$\textbf{FLAIR, University of Oxford}
\par 
$^*$Work undertaken while visiting FLAIR.
}

\contribution{We provide a reformulation of UED that is strongly concave in the adversary's strategy, allowing for easier convergence.}{\citet{dennis2021emergentcomplexityzeroshottransfer}'s initial UED work PAIRED uses a nonconvex-nonconcave objective, which is known to be unstable in training \citep{wiatrak2020stabilizinggenerativeadversarialnetworks}. Moreover, follow-up works such as \citep{chung2024adversarialenvironmentdesignregretguided} that improve PAIRED's level generator with generative models maintain this property.} 
\contribution{We provide convergence guarantees for any score function that is a zero-sum game over the policy's negative return (e.g. regret or negative return).}{Prior works in UED \citep{dennis2021emergentcomplexityzeroshottransfer,jiang2022replayguidedadversarialenvironmentdesign} assert guarantees if the UED game reaches a saddle point, but fail to guarantee convergence to one. We propose a method that provably converges.}
\contribution{We provide an empirical evaluation of our methods on current UED benchmarks, using relevant optimisation heuristics and by introducing a new score function that generalises the work of \citet{rutherford2024regretsinvestigatingimprovingregret} to general deterministic RL environments.}{Learnability \citep{rutherford2024regretsinvestigatingimprovingregret} requires a binary-outcome environment.}


\summary{For reinforcement learning agents to be deployed in high-risk settings, they must achieve a high level of robustness to unfamiliar scenarios. One approach for improving robustness is unsupervised environment design (UED), a suite of methods that aim to maximise an agent's generalisability by training it on a wide variety of environment configurations. In this work, we study UED from an optimisation perspective, providing stronger theoretical guarantees for practical settings than prior work. Whereas previous methods relied on guarantees \textit{if} they reach convergence, our framework employs a nonconvex-strongly-concave objective for which we provide a \textit{provably convergent} algorithm in the zero-sum setting. We empirically verify the efficacy of our method, outperforming prior methods on two of three environments with varying difficulties.}

\begin{document}

\makeCover  
\maketitle  

\begin{abstract}
For reinforcement learning agents to be deployed in high-risk settings, they must achieve a high level of robustness to unfamiliar scenarios. One approach for improving robustness is unsupervised environment design (UED), a suite of methods that aim to maximise an agent's generalisability by training it on a wide variety of environment configurations. In this work, we study UED from an optimisation perspective, providing stronger theoretical guarantees for practical settings than prior work. Whereas previous methods relied on guarantees \textit{if} they reach convergence, our framework employs a nonconvex-strongly-concave objective for which we provide a \textit{provably convergent} algorithm in the zero-sum setting. We empirically verify the efficacy of our method, outperforming prior methods on two of three environments with varying difficulties.\footnote{Our code is available at \url{https://github.com/nmonette/NCC-UED}.}
\end{abstract}

\section{Introduction}
Training reinforcement learning (RL) agents that are robust to a variety of unseen scenarios is an important and long-standing challenge in the field~\citep{morimoto2000robust,tobin2017domainrandomizationtransferringdeep}.One promising approach to addressing this problem is Unsupervised Environment Design (UED), where an adversary automatically proposes a diverse range of training tasks for an agent to learn on, based on its current capabilities~\citep{dennis2021emergentcomplexityzeroshottransfer,jiang2022replayguidedadversarialenvironmentdesign}. In this way, the agent gradually progresses from easy tasks (also called levels) to more difficult ones.

UED is generally posed as a two-player game between a level-selecting adversary and a level-solving agent. Most current methods use a variation of the \textit{minimax regret} approach, in which levels with high regret---meaning an agent performs far from optimal---are selected by an adversary for the agent to learn in~\citep{dennis2021emergentcomplexityzeroshottransfer, jiang2021prioritizedlevelreplay, jiang2022replayguidedadversarialenvironmentdesign, jiang2022groundingaleatoricuncertaintyunsupervised, parkerholder2023evolvingcurricularegretbasedenvironment}. Using regret as a  \textit{score} function is intuitive, since it shows how suboptimal a policy is on a specific level, i.e., how much it can still improve. However, despite some empirical success \citep{dennis2021emergentcomplexityzeroshottransfer,jiang2021prioritizedlevelreplay}, more recent works have presented a number of issues with the minimax regret formulation \citep{jiang2022groundingaleatoricuncertaintyunsupervised, beukman2024refiningminimaxregretunsupervised, rutherford2024regretsinvestigatingimprovingregret}, motivating the creation of new ways of determining the \textit{score} of a level. The canonical formulation of UED faces several other challenges, including lack of convergence guarantees and instability when searching for useful levels in a large space. 

In our work, we provide a reformulation of minimax regret as the expected regret when sampling levels from a categorical distribution. Using this reformulation, we provide a gradient-based algorithm that is \textit{provably convergent} to a locally optimal policy, relying on recent results \citep{lin2024twotimescalegradientdescentascent} in two-timescale gradient descent that ensure convergence when learning rates for minimiser (agent) and maximiser (adversary) are \textit{separated}. This stands in contrast to the existing paradigm that only has theoretical guarantees \textit{if} convergence is achieved.

Using regret as the score function is theoretically attractive because it maintains the zero-sum property between the agent and the adversary, and therefore allows for stronger guarantees. 
However, high-regret levels may not lead to efficient learning; moreover, computing regret is generally intractable, as it requires an optimal policy for each level. Inspired by this, \citet{rutherford2024regretsinvestigatingimprovingregret} use    \textit{learnability}~\citep{tzannetos2023proximalcurriculumreinforcementlearning} as an effective scoring function, but their formulation is limited to deterministic, binary-outcome domains. 
In our work, we generalise this score function to arbitrary deterministic settings, using the interpretation of learnability as the variance of agent success.
We finally develop a practical UED algorithm, NCC (see Figure \ref{fig:alg}), using our reformulation and generalised score function which obtains competitive results in three challenging domains. In this work, we provide a stepping stone towards more robust RL algorithms by creating a theoretically sound optimisation framework that offers competitive empirical results.

\begin{figure}[!t]
    \centering
    \includegraphics[width=\linewidth]{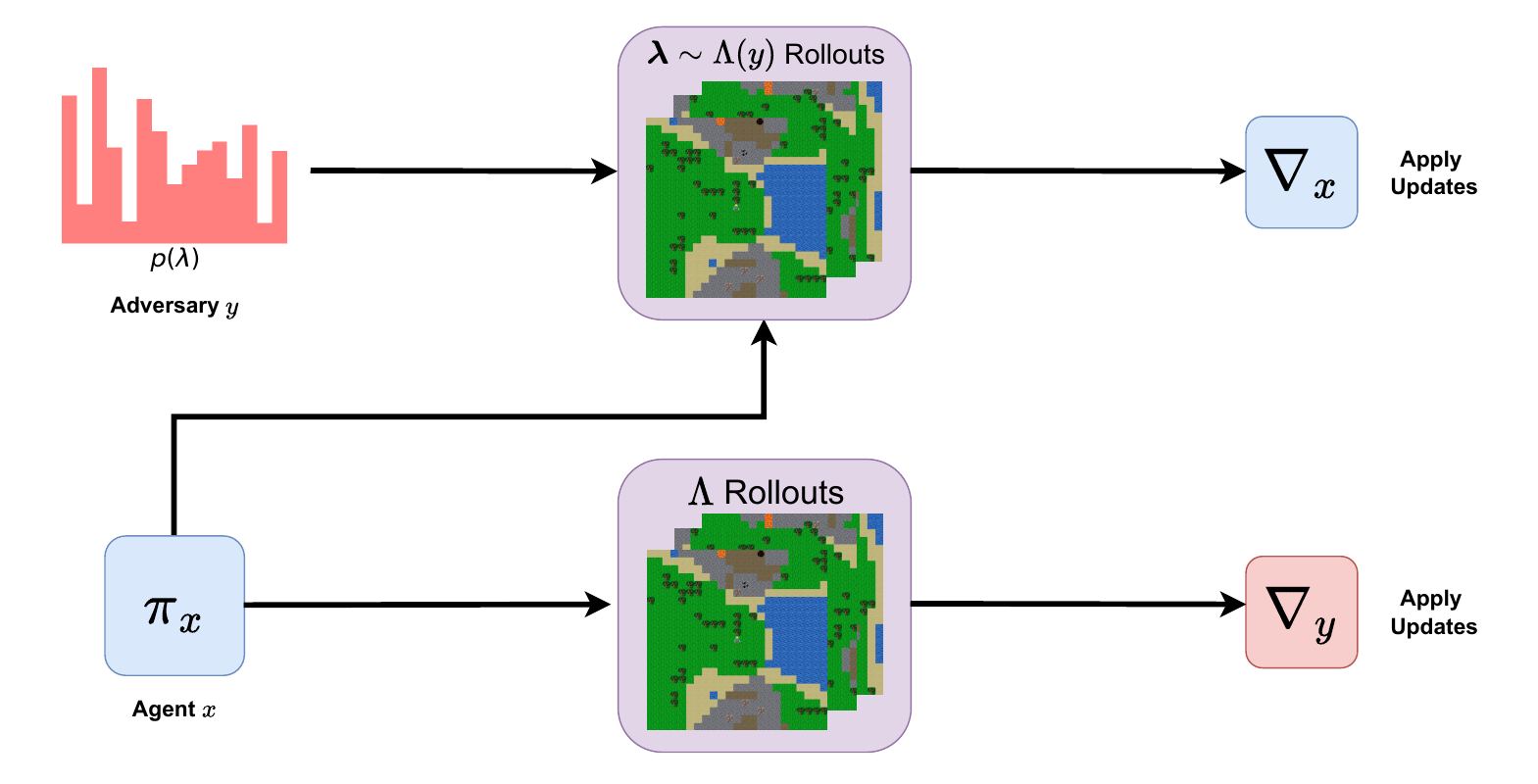}
    \caption{A visual representation of our training loop, which has simultaneous updates for the agent $x$ and the adversary $y$. The agent's update is trained on levels $\boldsymbol{\lambda}$ sampled from $y$, and the adversary's update is computed with scores computed from the policy on all levels from level-buffer $\Lambda$. We use the \textit{Craftax} environment from \citet{matthews2024craftaxlightningfastbenchmarkopenended} for illustration.\vspace{-0.5cm}}\label{fig:alg}
\end{figure}

\section{Background}
\subsection{Underspecified POMDP's} The underlying theoretical framework behind UED  is the Underspecified Partially Observable Markov Decision Process (UPOMDP). The UPOMDP \citep{dennis2021emergentcomplexityzeroshottransfer} describes an environment with level space $\mathcal{L}$, such that we train over some subset $\Lambda \subseteq \mathcal{L}$, where each parametrisation $\lambda \in \Lambda$ represents a POMDP. A UPOMDP is defined by the tuple $(\mathcal{L}, \mathcal{S}, \mathcal{O}, \mathcal{A}, r, P,  \rho, \gamma)$, where $\mathcal{S}$ is the state space, $\mathcal{O}$ is the observation space, and each $o \in \mathcal{O}$ is typically a limited view of the global state of the environment. The action space is $\mathcal{A}$, and the reward function is defined as $r: \mathcal{L} \times \mathcal{S} \times \mathcal{A} \to \mathbb{R}$. The transition probability function $P$ has a varying definition depending on the environment dynamics, but for the discrete-state case we write $P: \mathcal{L} \times \mathcal{S} \times \mathcal{A} \to \Delta(\mathcal{S})$, where $\Delta(\mathcal{S})$ is the probability simplex of size equal to the cardinality of set $\mathcal{S}$. Finally, $\gamma$ denotes the discount factor and $\rho: \mathcal{L} \to \Delta(\mathcal{S})$ denotes the initial state distribution function.

Consider the agent's parameter space $\mathcal{X}$, and policy $\pi: \mathcal{X} \times \mathcal{O} \to \Delta(\mathcal{A})$, where the policy may condition on a hidden state $h$ in the partially observable setting. The expected discounted return of the agent on a given level is $J: \mathcal{L} \times \mathcal{X} \to \mathbb{R}$. Moreover, we define the general objective of our agent to be the following, where $\Lambda(y)$ is some distribution over levels parametrised by $y$: 
\begin{equation}
    \max_{x \in \mathcal{X}} ~\mathbb{E}_{\lambda \sim \Lambda(y)} \bigg[J(\pi_x, \lambda)\bigg] \text{.}
\end{equation}
\subsection{Learning in Games}
A game is a scenario where there are agents interacting with each other by taking \textit{actions}, typically under the assumption that each agent is trying to maximise their own utility.
\paragraph{Solutions of Games}\textit{Learning} a game generally involves optimising for an equilibrium point between the players. Typically, the hope is that the players will achieve a \textit{Nash Equilibrium}~\citep[NE]{Nash1951}, which requires that neither player can unilaterally deviate their strategy to obtain a better utility. In such equilibria, players are \textit{robust} to changes to the opponent's strategy. Hence, the robustness guarantees of prior UED works \citep{dennis2021emergentcomplexityzeroshottransfer, jiang2022replayguidedadversarialenvironmentdesign} are derived under an assumption that their systems have converged to a NE. 
\paragraph{First-Order Nash Equilibria} Following \citet{nouiehed2019nash}, we consider the solution concept of the ($\epsilon$-approximate) first-order Nash Equilibrium. For $\epsilon \ge 0$, unconstrained $x$, and $y$ constrained to $\mathcal{Y}$, a first-order NE $(x^*,y^*)$ of the objective $\min_x \max_{y \in \mathcal{Y}} f(x, y)$,  is defined by
\begin{equation}\label{eq:fo_nash}
    \begin{aligned}
        &\|\nabla_x f(x^*, y^*)\| \le \epsilon \\
        \max_{y \in \mathcal{Y}}~&\langle \nabla_y f(x^*, y), y - y^*\rangle \le \epsilon \,
        \text{s.t. } \|y - y^*\| \le 1\,.
    \end{aligned} 
\end{equation}
An interpretation of the first-order NE is more clear when fixes on variable: neither $x$ or $y$ are able to further optimise w.r.t. $f$ via first-order gradient dynamics except for by some (small) distance $\epsilon$. 
\subsection{Game Theory and UED} 
\paragraph{Zero-sum UED} Prior works frame UED as a zero-sum game between an agent (the policy interacting with the environment) and a level-generating adversary \citep{dennis2021emergentcomplexityzeroshottransfer, jiang2022replayguidedadversarialenvironmentdesign}. The adversary aims to maximise the agent's \textit{score} on the levels. A common score function is \textit{regret}, defined as $\text{Reg}(\pi_x, \lambda) = J(\pi_*^\lambda, \lambda) - J(\pi_x,\lambda)$, for a level $\lambda \in \Lambda$ and its optimal policy $\pi_*^\lambda$.
An agent that maximises its expected return on a given level is equivalently minimising its regret, hence the a regret-maximising adversary is zero-sum with a return-maximising agent.

\paragraph{Minimax Regret} PAIRED \citep{dennis2021emergentcomplexityzeroshottransfer}, uses a generator parametrised by $y$ as their adversary. While not explicitly stated, the objective being optimised is
\begin{equation}
    \min_\pi \max_y~\mathbb{E}_{\lambda \sim  \Lambda(y)} \bigg[ \text{Reg}(\pi, \lambda) \bigg].
\end{equation}

In their analysis, \citet{dennis2021emergentcomplexityzeroshottransfer} construct a \textit{normal form} game (i.e., a game represented by a payoff matrix for each player), it is assumed that the action space of the agent is the (finite) set of possible deterministic policies. Practically, however, PAIRED trains a stochastic neural network-based policy via PPO \citep{schulman2017proximalpolicyoptimizationalgorithms}, and is not deterministic during training. In addition, due to the use of nonlinear neural networks for both the policy and the generator, the objective is in fact nonconvex-nonconcave. 
Hence, the normal-form construction (which is convex-concave) is not a reasonable representation of the UED problem. Instead, we argue that UED should be viewed as a min-max optimisation problem over the parameters of the agent and adversary. 

The theoretical results of \citet{dennis2021emergentcomplexityzeroshottransfer} results hold only \textit{at} Nash Equilibrium, but there is no guarantee that this NE will be reached. In fact, the system is unlikely to converge to a NE due to the nonconvex-nonconcave objective, which is not well-understood in the optimisation literature without additional unmet structural assumptions \citep{mertikopoulos2018optimistic,jin2020local, cai2024ncnc}. Secondly, the minimax theorem does not hold on the account of the nonconvexity/nonconcavity of the optimised variables \citep{jin2020local}. We circumvent both issues by constructing a nonconvex-strongly-concave objective for UED and proving that the variables involved converge to a first-order NE without needing to invoke the minimax theorem for analysis.

\subsection{Choice of Score Function}
Beyond issues with the theoretical framework of minimax regret, regret is often not a practically viable choice as a score function. While regret incentivises the adversary to propose levels where the agent has much capacity to improve, these levels may not lead to optimal learning, and in fact may not be conducive to learning at all. Regret also relies on the optimal policy, which is generally not available. There is also the \textit{regret stagnation} problem \citep{beukman2024refiningminimaxregretunsupervised}, where due to some stochasticity or partial observability in an environment, the regret is not reducible below some non-minimal value.\footnote{For example, if return $R(\cdot, \lambda) \in [0, 1]$, and cannot be increased past $0.7$, $\text{Reg}(\cdot, \lambda)$ is irreducible below $0.3$.} When irreducible, regret is no longer representative of policy learning potential.

One prevailing alternative to regret is the \textit{learnability} of a level \citep{rutherford2024regretsinvestigatingimprovingregret}. Assuming a binary-outcome setting, where the return $R(\tau, \lambda)$ is in $\{0, 1\}$ for all trajectories $\tau$ and levels $\lambda$, learnability is defined as $s(\pi_x, \lambda) = p(1 - p)$, where $p$ is the agent's solve rate $p = \E_{\tau \sim \pi_x(\lambda)} \left[ R(\tau, \lambda) \right]$. For binary outcomes, this can be rewritten as the variance of returns as follows:
\begin{align}\label{eq:learn}
    s(\pi_x, \lambda) = \text{Var}_{\tau \sim \pi_x(\lambda)}\bigg[R(\tau, \lambda) \bigg] \,.
\end{align}
Learnability has a number of interpretations that are explored in \citet{tzannetos2023proximalcurriculumreinforcementlearning} and \citet{ rutherford2024regretsinvestigatingimprovingregret}, but the variance interpretation is intuitive in the sense that levels with low variance of returns are either too difficult or too easy, and should not be prioritised during training.  

\section{Related Work}\label{limits}

UED algorithms can generally be categorised into either \textit{sampling}-based and \textit{generative}-based approaches. The former assume access to some function that samples levels; often, this is the uniform distribution over levels $\mathcal{L}$. The latter instead learn a level-generating model (either via RL~\citep{dennis2021emergentcomplexityzeroshottransfer} or self-supervised learning~\citep{garcin2024dred}).

The simplest UED method, called Domain Randomisation (DR), directly trains an agent on this uniform distribution~\citep{tobin2017domainrandomizationtransferringdeep}. 
In simple environments, DR can perform competitively to other UED algorithms~\citep{coward2024jaxuedsimpleuseableued}, but falls behind when the environment becomes more complex~\citep{matthews2024kinetixinvestigatingtraininggeneral}.
Other sampling-based approaches, like Prioritised Level Replay \citep[PLR]{jiang2021prioritizedlevelreplay, jiang2022replayguidedadversarialenvironmentdesign} and Sampling For Learnability \citep[SFL]{rutherford2024regretsinvestigatingimprovingregret} prioritise training on levels with a high \text{score} (e.g., regret, learnability, etc.). 
However, empirically these methods tend to work best when the agent is \textit{also} trained on some levels from the unfiltered DR distribution, in addition to the high-scoring ones.

While these sampling methods have achieved impressive empirical results, they do not enjoy convergence guarantees due to the use of heuristics in place of gradient-based optimisation. 
In our work, we establish a two-player game objective that is efficiently optimisable with gradients. Moreover, we provide convergence guarantees for regret, but defer theoretical considerations of learnability to future work, as the general-sum setting induces a significant departure from our current method.

Generative approaches such as PAIRED have also seen success, particularly when replacing the RL level generator with a deep generative model \citep{azad2023clutr, li2024enhancinghierarchicalenvironmentdesign, garcin2024dred, chung2024adversarialenvironmentdesignregretguided}. 
However, the optimisation objective is nonconvex-nonconcave due to the neural adversary, and thus is known to have issues with instability and convergence \citep{wiatrak2020stabilizinggenerativeadversarialnetworks}. 
Our method does not use a deep generative model, but unifies the generative and sampling approaches by \textit{learning} the sampling distribution with gradient optimisation.
\section{Method}
We reformulate UED as a regularised game over expected score, and derive optimisation guarantees from this formulation. We then break the theoretical assumptions for empirical reasons in Section \ref{impl}.

\subsection{Core Optimisation Problem} Consider $s: \mathcal{X} \times \mathcal{L}  \to \mathbb{R}$, where $s(\pi_x, \lambda)$ denotes the policy's score on level $\lambda$. We abuse notation to rewrite any level-wise function of $\lambda$ as a function of $\Lambda$ in  \textbf{bold} to denote the vector of the function evaluated at all levels $\lambda \in \Lambda$ (e.g. $\boldsymbol{s}(\pi_x, \Lambda)$ is the \textit{score vector}). Additionally, we define $\mathcal{Y} := \Delta(\Lambda)$ as the feasibility set for $y$. Motivated by the unstated formulations of \citet{dennis2021emergentcomplexityzeroshottransfer} and \citet{jiang2022replayguidedadversarialenvironmentdesign}, we establish the expected score objective for UED, similar to \citet{qian2018robustoptimizationmultipledomains}, which is linear in the adversary's strategy: 
\begin{align}\label{eq:reform}
    \min_{x \in \mathcal{X}} \max_{y \in \mathcal{Y}}~\mathbb{E}_{\lambda \sim \Lambda(y)} \bigg[ s(\pi_x, \lambda) \bigg] = \min_{x \in \mathcal{X}} \max_{y \in \mathcal{Y}}~ y^T \boldsymbol{s}(\pi_x, \Lambda) \,.
\end{align}
Extending the \textit{soft UED} framework of \citet{chung2024adversarialenvironmentdesignregretguided}, we add an entropy regularisation term $\mathcal{H}(y) = -y^T \log y$ to the objective of our adversary:
\begin{equation}\label{eq:minmax-entropy}
    \min_{x \in \mathcal{X}} \max_{y \in \mathcal{Y}}~ f(x, y) \coloneqq \min_{x \in \mathcal{X}} \max_{y \in \mathcal{Y}}~ y^T \boldsymbol{s}(\pi_x, \Lambda) + \alpha \mathcal{H}(y) \,,
\end{equation}
where $\alpha > 0$ is a temperature coefficient. Our justification is twofold: first, as per \citet{chung2024adversarialenvironmentdesignregretguided} the agent needs to train on several different levels at each iteration, requiring the adversary's distribution to have greater entropy instead of collapsing to a single level whose score is largest. Second, entropy regularisation ensures that $f$ is strongly concave in $y$, guaranteeing best-iterate convergence (Theorem \ref{th}). For general score functions, the optimisation problem is given by:
\begin{equation}\label{eq:genform}
\begin{aligned}
    \min_{x \in \mathcal{X}}~f(x, y) = -y^T\boldsymbol{J}(\pi_x, \Lambda) \ \ , \ \ \max_{y \in \mathcal{Y}}~g(x, y) = y^T \boldsymbol{s}(\pi_x, \Lambda) + \alpha \cH(y) \,.
\end{aligned}
\end{equation}
Because UED conventionally uses a nonlinear neural network to parameterise its policy (and value function), we have an objective that is nonconvex in $x$ and strongly concave in $y$. \citet{lin2024twotimescalegradientdescentascent} have shown that under certain assumptions, we can guarantee (best-iterate) convergence using \textit{two-timescale} stochastic gradient descent-ascent, which assumes a separation of learning rates. 
Using these results, we propose \textbf{N}on\textbf{C}onvex-\textbf{C}oncave optimisation for UED (NCC) as a theoretically based method for optimisation in the UED setting.

\subsection{Method for Optimisation}
To perform gradient based optimisation for the adversary, we construct the score vector $\mathbf{s}$ before each iteration of RL training to construct $y$'s gradient. For the adversary, we perform projected gradient ascent constrained to the probability simplex, and for $x$ we perform unconstrained gradient descent. The training loop is summarised in Algorithm \ref{alg:theory} (illustrated in Figure \ref{fig:alg}), where $\mathcal{P}_\mathcal{X}(\cdot)$ represents the Euclidean projection onto set $\mathcal{X}$. For learning rates $\eta_y \gg \eta_x$, and stochastic gradient estimators $\hat{F}$ and $\hat{G}$ defined in \cref{eq:G,eq:H}, our update rule can be summarised as the following:
\begin{equation}
    x^{t+1} = x^t - \eta_x \cdot \hat{F},~~y^{t+1} =\mathcal{P}_\mathcal{Y} \big(y^t + \eta_y \cdot \hat{G} \big) \,.
\end{equation} 
In theory, NCC uses a single stochastic gradient step of $x$ using the gradient estimator in Equation \eqref{eq:G}, and relies on a static buffer.  While we make typical assumptions about the policy architecture and policy gradient estimator for the sake of theoretical analysis, we do not necessarily use these in practice due to worsened empirical performance (see Appendix \ref{sec:comparison} for more details). 

\begin{algorithm}[h]
    \caption{Nonconvex-concave Optimisation for UED}

    \begin{algorithmic}\label{alg:theory}
        \REQUIRE{Initial policy $x^{0}$, distribution $y^0 = \frac{1}{|\Lambda|}\mathbf{1}$, stepsizes $\eta_x, \eta_y$}, level set $\Lambda$.
        \FOR{$t = 0, 1, \ldots$} 
            \STATE Sample batch of training levels $\boldsymbol{\lambda} \sim \Lambda(y^{t})$
            \STATE Construct score vector $\mathbf{s} = \boldsymbol{s}(\pi_x, \Lambda)$
            \STATE $x^{t+1} = x^t - \eta_x \cdot \hat{F}(x^t, y^t; \boldsymbol{\lambda})$ with $\hat{F}$ defined in Equation \eqref{eq:G}
            \STATE $y^{t+1} = \mathcal{P}_{\mathcal{Y}} \bigg(y^{t} + \eta_y \cdot\hat{G}(x^{t}, y^{t}; \mathbf{s})\bigg)$ with $\hat{G}$ defined in Equation \eqref{eq:H}
           
        \ENDFOR\\
        \STATE \textbf{return} Best-iterate policy parameters $x^*$
    \end{algorithmic}
\end{algorithm}

\section{Convergence Results}\label{sec:guarantees}
In order to guarantee convergence in the zero-sum setting, we use two-timescale stochastic gradient descent-ascent \citep{lin2024twotimescalegradientdescentascent}, as given by Algorithm \ref{alg:theory}, to find an approximate solution to the optimisation problem defined by Equation \eqref{eq:minmax-entropy}. We first make the necessary assumptions and definitions, and then state our main theorem.

\subsection{Preliminaries and Assumptions}
\paragraph{Notation} We denote $F = \nabla_x f(x, y)$, $G = \nabla_y f(x, y)$, and $H = \nabla f = (F, G)$. Moreover, we let $N = \abs{\Lambda}$ be the number of levels and $\lambda_i$ be the $i$-th level. We denote trajectories by $\tau^t = (o^t, a^t, r^t)$, with $\pi_x(\tau^t) \coloneqq \pi_x(a^t|o^t)$ for short. We write $\Psi^t(\tau, \lambda)$ for the estimator used in policy gradients, typically taken to be the discounted Q-function $Q_{\pi_x}^t(\tau, \lambda) = \gamma^t Q_{\pi_x}(s^t, a^t, \lambda)$, advantage function $A_{\pi_x}^t(\tau, \lambda) = \gamma^t A_{\pi_x}(s^t, a^t, \lambda)$ or return $R^t(\tau, \lambda) = \sum_{h = t}^t \gamma^h r^h$ from time $t$ onwards. The discount ensures that the policy gradient is a true gradient, as clarified by \citet{nota2020policy}.

In order to prove convergence guarantees, we need to make some basic regularity assumptions on the UED and policy network architecture. The first is standard for returns to be finite and for regularity conditions to hold, while the second can be guaranteed by clipping, normalising or otherwise bounding network weights, as discussed in Appendix \ref{app:proofs}.

\begin{Assumption}
\label{as:1}
    The number of levels $N$ and the longest episode length $T$ are finite, and the state and reward spaces are bounded. We consequently write $R_{*} = \max_{\tau, \lambda} \abs{R(\tau, \lambda)} < \infty$ for the largest absolute discounted return across trajectories and levels.
\end{Assumption}

\begin{Assumption}
\label{as:2}
    The agent policy $\pi_x$ is a $\zeta$-greedy policy parameterized by an $L$-Lipschitz and $K$-smooth function with parameters $x$. The adversary distribution is constrained to the $\xi$-truncated probability simplex, namely $\cY = \Delta_\xi(\Lambda) \coloneqq \{ y \in \Delta(\Lambda) \mid y_i \geq \xi \ \forall i\}$.
\end{Assumption}

\paragraph{Gradient Estimators} For the purpose of analysis, we generalize REINFORCE \citep{williams1992reinforce} to the UED setting by defining an unbiased estimator for our agent's gradient $F$ as an expectation over $N$ levels $\lambda_i$ sampled from $\Lambda(y)$, with a batch size of $M$ trajectories for each level:
\begin{equation}\label{eq:G}
    \hat{F}(x, y) = - \frac{1}{NM} \sum_{i, j} \sum_{t=0}^T \nabla_x \log \pi_x(\tau_{ij}^t, \lambda_i) \Psi^t(\tau_{ij}, \lambda_i) \,,
\end{equation}
where $\lambda_i \sim \Lambda(y)$ and trajectories $\tau_{ij} \sim \pi_x(\lambda_i)$ are sampled independently. For the adversary, the unbiased estimator gradient is similarly given by
\begin{equation}\label{eq:H}
    \hat{G}(x, y) = \hat{\boldsymbol{s}}(\pi_x, \Lambda) + \alpha \nabla_y \mathcal{H}(y) \,,
\end{equation}
where $\hat{s}$ is the empirical score vector, given by $\hat{s}(\pi_x, \lambda_i) = -\frac{1}{M} \sum_j R(\tau_{ij}, \lambda_i)$ for $s = -J$ and $\hat{s}(\pi_x, \lambda_i) = \max_{\tau} R(\tau, \lambda_i) -\frac{1}{M} \sum_j R(\tau_{ij}, \lambda_i)$ for $s = \text{Reg}$.

\subsection{Convergence Guarantees}

\begin{Proposition}\label{prop:1}
    Under \cref{as:1,,as:2}, the estimator $\hat{H} = (\hat{F}, \hat{G})$ defined in \cref{eq:G,eq:H} has $\sigma^2_M$-bounded variance, where
    $$ \sigma^2_M = \frac{4 R_*^2}{M} \left( N + \frac{T^2 L^2}{\zeta^2} \right) \,. $$
    Moreover, the corresponding objective $f(x,y) = y^T \boldsymbol{s}(\pi_x, \Lambda) + \alpha \cH(y)$ defined in \cref{eq:minmax-entropy} is $\alpha$-strongly concave in $y$, Lipschitz, and $\ell$-smooth, where
    $$ \ell = \frac{TR_*}{\zeta} \left( TL^2 + K + \frac{L^2}{\zeta} + 2TL \right) + \frac{\alpha}{\xi} \,. $$
\end{Proposition}

\begin{proof}
    In Appendix \ref{app:proofs}.
\end{proof}

\begin{theorem}[Best-Iterate Convergence]\label{th}
    Under \cref{as:1,,as:2}, let $\ell$ and $\sigma \coloneqq \sigma_1$ be the constants defined in Proposition \ref{prop:1}, $\alpha$ the entropy temperature, and $\Delta = \max_y f(x^{0}, y) - \max_y f(x^*, y)$ the objective distance between initial and optimal policies. For learning rates $\eta_x = \Theta(\alpha^2/\ell^3)$ and $\eta_y = \Theta(1/\ell)$, and a batch size $M = \Theta(\max\{1, \sigma^2\ell/\alpha\ep^2\})$, Algorithm \ref{alg:theory} finds an $\ep$-stationary policy $\pi_{x^*}$, such that $\norm{ \nabla_x \max_y f(x^*,y) } < \ep$, in a number of iterations given by
    $$ O\left(\frac{\Delta \ell^3}{\alpha^2\ep^2} + \frac{2 \ell^3}{\alpha\ep^4}\right) \,. $$
\end{theorem}

\begin{proof}
    Apply Proposition \ref{prop:1} and \citet[Theorem 4.5]{lin2024twotimescalegradientdescentascent}, with $D \leq \sqrt{2}$ being the diameter of the $\xi$-truncated probability simplex $\cY$ and $\kappa = \ell/\alpha$ the condition number.
\end{proof}

We remark that while we are only concerned with finding a stationary policy $\pi_{x^*}$ in the UED setting, the corresponding optimal distribution $y^* = \argmax_{y \in \cY} f(x^*, y)$ can efficiently be computed via projected gradient ascent due to the strong concavity of $f$ in $y$. The resulting point $(x^*, y^*)$ meets the conditions of Equation \eqref{eq:fo_nash}, and is therefore an $\epsilon$-approximate first-order Nash Equilibrium.
\section{Practical Considerations}\label{impl}
To detail the practical extension of our algorithm, we first introduce the \textit{dynamic} level buffer, and then explain our proposed generalisation of the learnability score function from \citet{rutherford2024regretsinvestigatingimprovingregret}. Additional discussion of the differences between our practical and theoretical methods are in Appendix \ref{sec:comparison}, and the full practical algorithm is given in Algorithm \ref{alg:practice}.
\subsection{Searching the level space}\label{sec:dynamic}

Given that the environment has a large enough level space (i.e. $|\mathcal{L}| \gg |\Lambda|$), it has been demonstrated that intermittently sampling new levels and exchanging them for low-scoring levels in the level buffer is often necessary for good performance \citep{jiang2021prioritizedlevelreplay}. Inherently, if the level space contains a high proportion of irrelevant (low-scoring) levels, the initially-sampled $\Lambda$ would lead to a poor training process if it were kept static. Considering this intuition, alongside the empirical results of \citet{jiang2021prioritizedlevelreplay}, we consider the dynamic case in practice.

To implement such a dynamic buffer, we compute the scores of newly sampled levels at every training iteration, and update the buffer with the top $|\Lambda|$ scoring levels \textit{prior} to constructing the adversary's gradient. 

\subsection{Heuristics and General-Sum UED} \paragraph{General-Sum UED} In practice, we use the same gradient estimators for $x$ and $y$ regardless of score function, and we find this method with general-sum score functions can lead to performance gains. However, with score functions that are not zero-sum with the policy's negative return, our method lacks convergence guarantees. We note that the baselines that we test our method against (i.e., PLR \citep{jiang2021prioritizedlevelreplay}, DR \citep{tobin2017domainrandomizationtransferringdeep}, and SFL \citep{rutherford2024regretsinvestigatingimprovingregret}) also lack convergence guarantees. 

\paragraph{Generalised Learnability} We extend the learnability score function of \citealp{rutherford2024regretsinvestigatingimprovingregret}, to general deterministic domains. As in the binary-outcome case, we aim to prioritise levels of intermediate difficulty for the current policy. We start with the standard deviation of the returns for a given level. However, unlike Equation~\ref{eq:learn}, we cannot entirely rely on a variance metric as we empirically find that in several domains, levels where agents do very poorly have a high return variance. In order to bias scoring against levels that are not of intermediate difficulty, we scale the standard error values with a Gaussian over the mean return of the level buffer $\Lambda$. This reduces the score for levels of significant distance from the mean reward. While this approach could bias scores towards levels with a high range of reward outcomes, empirically we do not find this to be an issue.

Given a set of $M$ trajectories $\{\tau_i\}_{i=1}^M$ on level $\lambda \in \Lambda$, we compute the level-wise empirical mean $\mu_\lambda$ = $\frac{1}{M} \sum_{i=1}^M R(\tau_i, \lambda)$, overall mean $\mu = \frac{1}{N}\sum_{\lambda\in\Lambda}\mu_\lambda$, level-wise empirical variance $\sigma^2_\lambda = (\frac{1}{M} \sum_{i=1}^M R(\tau_i, \lambda)^2) - \mu_\lambda^2$, and overall variance $\sigma^2 =(\frac{1}{N}\sum_{\lambda} \mu_\lambda^2) - (\frac{1}{N}\sum_{\lambda} \mu_\lambda)^2 $. Using the Gaussian probability density function $\mathcal{N}(\cdot | \mu, \sigma^2)$, we get the generalised learnability score:
\begin{equation}
    s(\pi_x, \lambda) = \sigma_\lambda\cdot \mathcal{N}(\mu_\lambda | \mu,  \sigma^2) \,.
\end{equation}
In Appendix~\ref{sec:gen_learn_supp} we repeat the score function analysis of \citet{rutherford2024regretsinvestigatingimprovingregret}, demonstrating the score function's effectiveness on Minigrid. Generalised learnability thus allows for the use of generalised SFL (``Gen-SFL'' in Figure \ref{plot:craftax}), which obtains superior performance on Craftax.

\section{Experiments}\label{sec:exps}
Alongside our theoretical considerations, our method outperforms contemporary work on UED benchmarks after being extended to a practical algorithm. In this section we detail the choice of benchmarks and provide an experimental evaluation of our method's performance.

\subsection{Experimental Setup}
We test our method on benchmarks from \citet{rutherford2024regretsinvestigatingimprovingregret} and report results on Minigrid \citep{cb2023minigrid}, using the implementation from \citet{coward2024jaxuedsimpleuseableued}, and XLand-Minigrid \citep{nikulin2023xlandminigrid}. We omit JaxNav, as the single-agent setting's results are highly saturated and the multi-agent setting introduces additional optimisation challenges. We refer the reader to \citet{rutherford2024regretsinvestigatingimprovingregret} for more details on the individual environments, noting only that Minigrid is the only environment with a regret oracle. Additionally, we show that our method can obtain competitive performance on a more complex benchmark, \textit{Craftax} \citep{matthews2024craftaxlightningfastbenchmarkopenended}.

In all provided plots, we show our contributions in \textbf{bold} font, and suffix ``NCC'' with the particular score function used (see \citet{rutherford2024regretsinvestigatingimprovingregret} for a further discussion of such score functions). ``Reg'', ``Learn'', and ``PVL'' correspond to regret, learnability, and positive value loss respectively. PVL is an approximation of regret when the latter is unknown, as in XLand-Minigrid and Craftax. For Minigrid, learning curves overlap significantly and we use a bar plot for visual clarity, whereas for XLand-Minigrid and Craftax we plot evaluation results over the course of training to additionally compare sample efficiency.

Experiments were written in JAX \citep{jax2018github} and we discuss experimental details in Appendix \ref{time}. Results are averaged across 10 seeds, and standard error from the mean is displayed in the plots. All experiments use PPO \citep{schulman2017proximalpolicyoptimizationalgorithms} as the RL algorithm of choice.

We largely use the hyperparameters from \citet{rutherford2024regretsinvestigatingimprovingregret} and \citet{matthews2024craftaxlightningfastbenchmarkopenended}, detailed in Appendix \ref{sec:hypers}. Moreover, we demonstrate the diversity of levels proposed by our method throughout training in Appendix \ref{sec:difficulty}.

\subsection{Results}
\begin{figure}[h]
    \centering
        \includegraphics[width=0.7\linewidth]{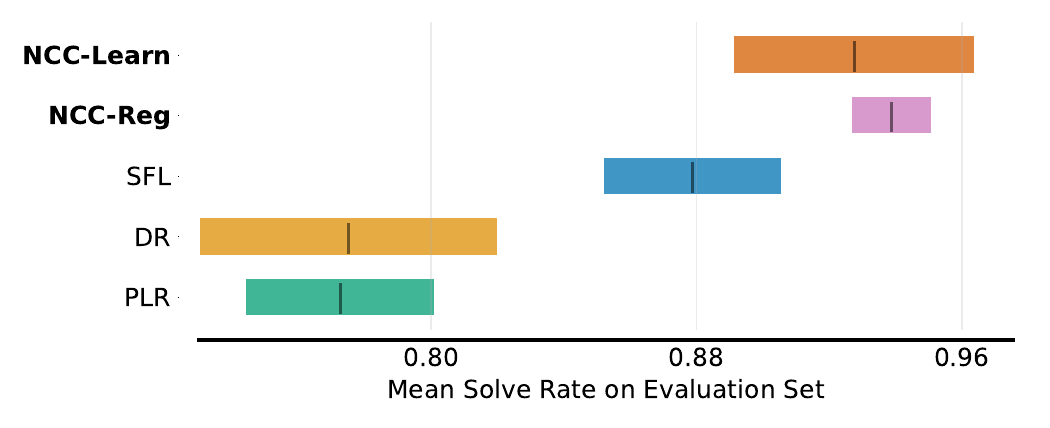} 
    \caption{Mean solve rates with standard error bars on Minigrid, a common UED testbed.\vspace{-0.25cm}}
    \label{plot:minigrid}
\end{figure}

\paragraph{Minigrid}
We observe improved performance on the Minigrid holdout set with 60 walls. For readability, we use a bar plot using recommendations from \citep{agarwal2021deep}, but with mean and standard error for consistency. Additionally, we use Minigrid to highlight our method's efficacy given a regret oracle, and thus do not test NCC-PVL.\footnote{We did use PLR-MaxMC instead of PLR-Reg however, because we did not see a improvement in empirical performance.} Despite the wisdom from \citet{rutherford2024regretsinvestigatingimprovingregret} that learnability improves learning, we hypothesise that due to the zero-sum nature of regret it may still lead to more robust policies in the end.

\begin{figure}[htbp]
    \centering
    \begin{subfigure}[b]{0.49\textwidth}
        \centering
        \includegraphics[width=\linewidth]{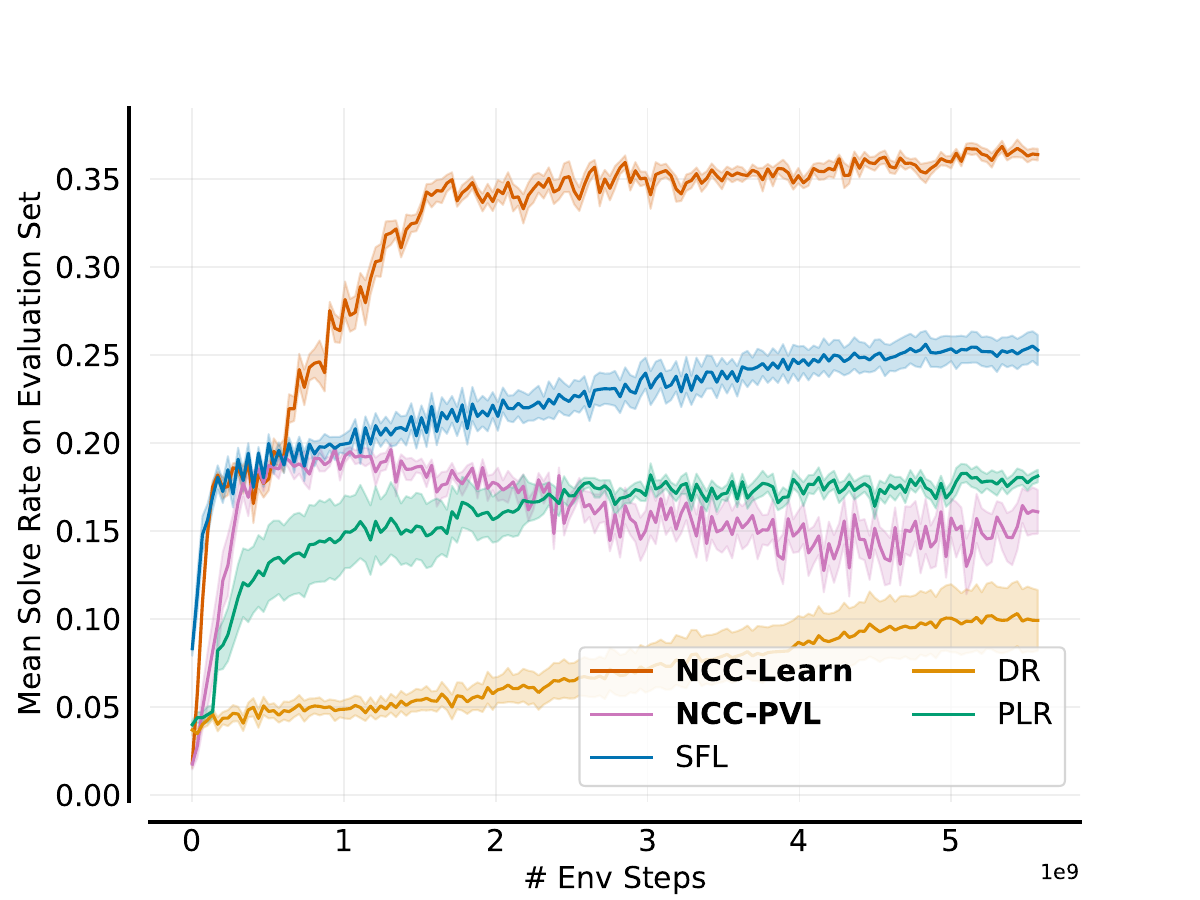}
        \caption{XLand-Minigrid}
    \end{subfigure}
    \hfill
    \begin{subfigure}[b]{0.49\textwidth}
        \centering
        \includegraphics[width=\linewidth]{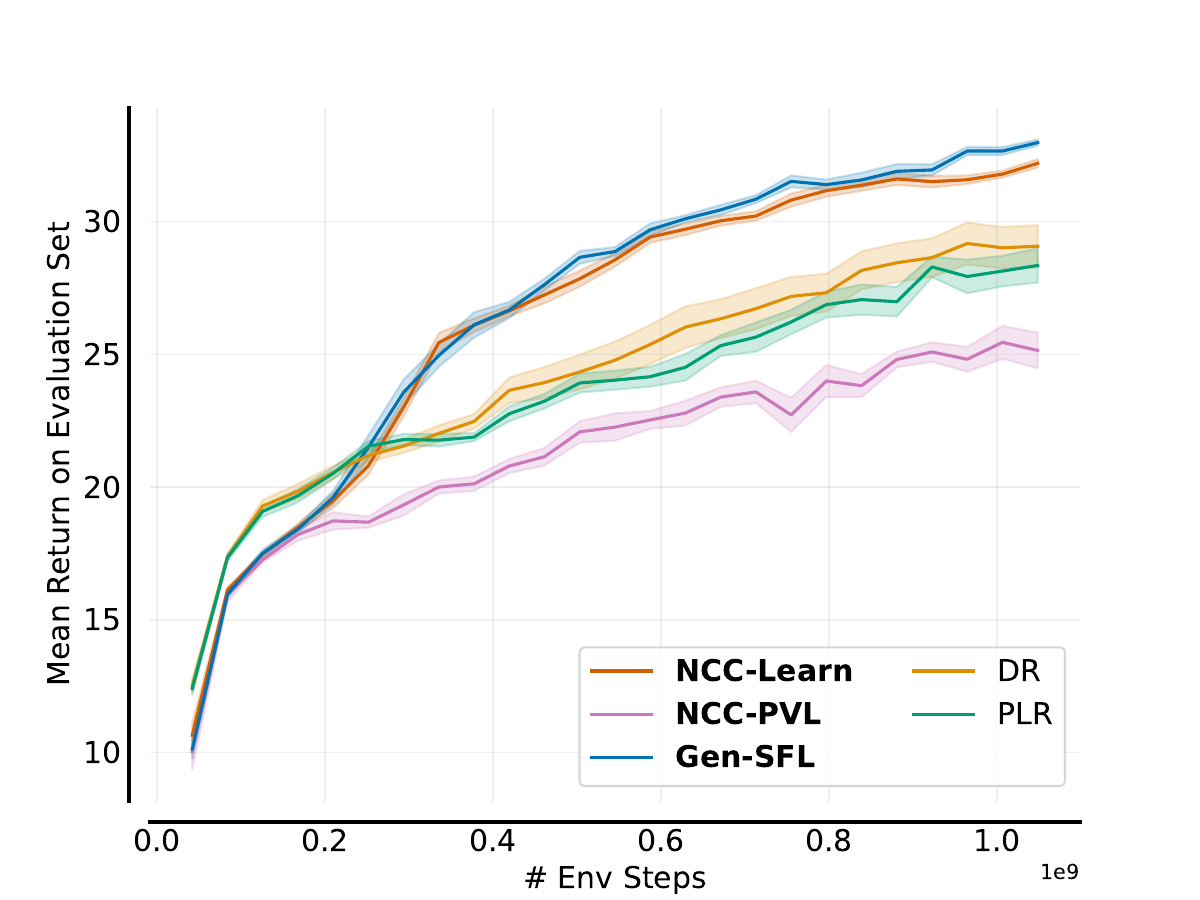} 
        \caption{Craftax}
    \end{subfigure}
    \caption{Performance on more difficult benchmarks.}\label{plot:craftax}
\end{figure}

\paragraph{XLand-Minigrid}
Our most significant improvement from prior work is in XLand-Minigrid. We note that out of the given testbeds, this environment has results that are less saturated, and thus leaves more room for improvement. NCC obtains a considerably improved solve rate compared to prior works, although we remark that this is only the case when using learnability as the score function.

\paragraph{Craftax} We use our new generalisation of learnability to outperform the highest-performing UED baseline from \citet[PLR-MaxMC]{matthews2024craftaxlightningfastbenchmarkopenended}. We remark that we find performance is stronger in Craftax with a \textit{static} buffer, and we highlight this to mention that in environments with a higher density of ``good'' levels, it may not be necessary to use a dynamic buffer.\footnote{In Craftax, the DR distribution of levels is the same as the evaluation distribution.} 
We attribute similar performance across generalised SFL and NCC with learnability to the algorithms' shared emphasis on levels with high learnability. See Appendix \ref{sec:craftax_dets} for additional implementation details for Craftax experiments.

\begin{figure}[H]
    \centering
    \includegraphics[width=0.5\linewidth]{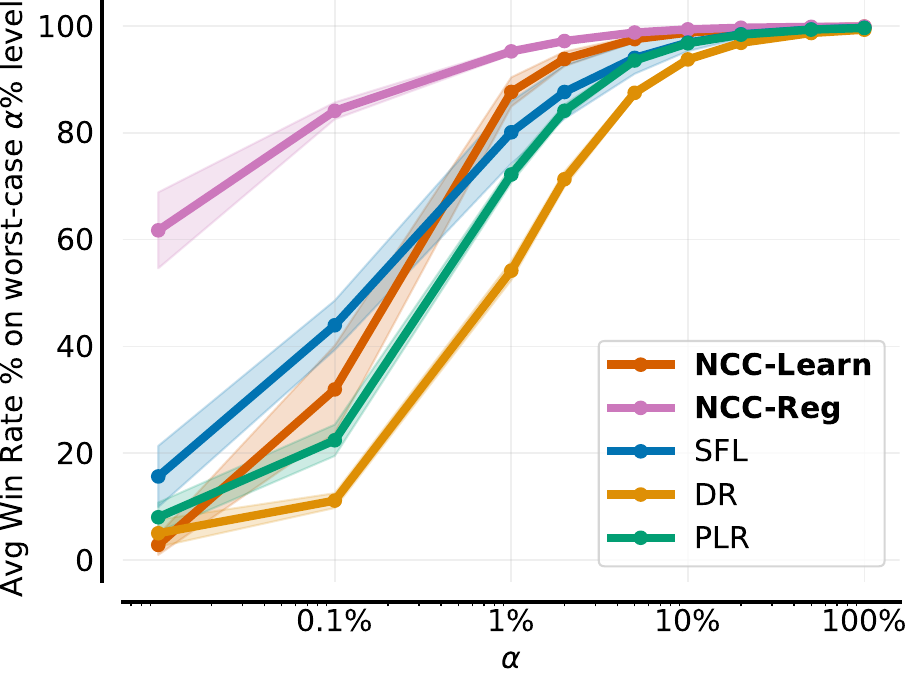}
    \caption{$\alpha$-CVaR evaluation performance on Minigrid.}
    \label{fig:grid_cvar}
\end{figure}
\paragraph{Robustness Evaluation} We also perform the $\alpha$-CVaR evaluation protocol from \citet{rutherford2024regretsinvestigatingimprovingregret}, which evaluates policies from 10 seeds per method on the $\alpha\%$ worst-case levels which are still solvable. We find that in XLand-Minigrid and Craftax, the robustness evaluations roughly match the experimental outcomes. However, in Minigrid (where regret is tractable) we find that NCC with regret strongly outperforms any other method, indicating NCC significantly impacts the robustness of the policy. We display our Minigrid result in Figure \ref{fig:grid_cvar}, and defer the other plots to Appendix \ref{sec:cvar}.
\section{Future Work}
Our work obtains best-iterate convergence guarantees, which is commonplace in nonconvex minimax optimisation \citep{lin2024twotimescalegradientdescentascent, kalogiannis2024learningequilibriaadversarialteam}. 
However, we leave the question of the more desirable last-iterate convergence property \citep{daskalakis2020lastiterateconvergencezerosumgames, constrained2021lei} to future work.

Moreover, while it may be possible to analyse the general-sum UED setting under the lens of bilevel optimisation \citep{hong2022twotimescaleframeworkbileveloptimization}, we would suggest that future work investigates practical and more scalable ways to produce convergent methods when the zero-sum condition is not met---as in the case when using learnability.
Finally, considering the emergence of analysis of more sophisticated reinforcement learning algorithms like PPO \citep{kuba2022mirror}, another promising direction for future work is analysing the practical variant of our algorithm.

\section{Conclusion}
In this paper, we examine the connection between UED and optimisation theory, thereby developing a new framework for optimisation within UED. This framework unlocks the area's first convergence guarantees, and suggests a new method that produces more robust policies. 
Following prior work, we use our theoretical analysis to develop a practical algorithm that obtains strong results on robustness evaluations. 
Ultimately, we believe that our work provides a gateway to the creation of practical robust RL methods with guarantees under reasonable assumptions, which will become increasingly important as UED gains prevalence.

\newpage


\appendix

\section{Proof of Proposition \ref{prop:1}}\label{app:proofs}

Before proving the proposition, we briefly discuss Assumption \ref{as:2}. In conjunction with Assumption \ref{as:1}, a simple and sufficient condition for it to hold is for $\pi_x$ to be parameterized by a neural network with bounded weights, composed of any number of fully-connected, convolutional, max-pooling, recurrent (vanilla / gated / LSTM) layers, dropout, batch normalization and smooth activation functions including Sigmoid, Softmax, Tanh, ArcTan, ELU, SELU, GELU, SoftPlus, Softsign \citep{lipschitz2018virmaux}. The condition that weights be bounded may seem restrictive, but (1) typically holds in practice because of weight regularisation, (2) can easily be ensured by clipping weights to an (arbitrarily large) box, which would only alter experimental results if gradients explode, and (3) can also be ensured by normalization across weight matrix rows, which in some cases may improve training stability at a minor cost in performance \citep{miller2019lstm}.

\begin{custom}{\ref{prop:1}}\label{proof:1}
    Under \cref{as:1,,as:2}, the estimator $\hat{H} = (\hat{F}, \hat{G})$ defined in \cref{eq:G,eq:H} has $\sigma^2_M$-bounded variance, where
    $$ \sigma^2_M = \frac{4 R_*^2}{M} \left( N + \frac{T^2 L^2}{\zeta^2} \right) \,. $$
    Moreover, the corresponding objective $f(x,y) = y^T \boldsymbol{s}(\pi_x, \Lambda) + \alpha \cH(y)$ defined in \cref{eq:minmax-entropy} is $\alpha$-strongly concave in $y$, Lipschitz, and $\ell$-smooth, where
    $$ \ell = \frac{TR_*}{\zeta} \left( TL^2 + K + \frac{L^2}{\zeta} + 2TL \right) + \frac{\alpha}{\xi} \,. $$
\end{custom}

\begin{proof}
We first derive the bounds for $M = 1$. Recall \cref{eq:G,eq:H} from the main text:
\begin{align*}
   \hat{F}(x, y) &= - \frac{1}{N} \sum_{i} \sum_{t=0}^T \nabla_x \log \pi_x(\tau_{i}^t, \lambda_i) \Psi^t(\tau_{i}, \lambda_i) \,, \\
    \hat{G}(x, y) &= \hat{\boldsymbol{s}}(\pi_x, \Lambda) + \alpha \nabla_y \mathcal{H}(y) \,, 
\end{align*}
where $\lambda_i \sim \Lambda(y)$ and $\tau_i \sim \pi_x(\lambda_i)$ are sampled independently for each level $i$, and $\hat{s}$ is the empirical score vector given by $\hat{s}(\pi_x, \lambda_i) = - R(\tau_{i}, \lambda_i)$ for $s = -J$ and $\hat{s}(\pi_x, \lambda_i) = \max_{\tau} R(\tau, \lambda_i) - R(\tau_{i}, \lambda_i)$ for $s = \text{Reg}$. Finally, denote $z = (x,y)$ for joint parameters.

\textbf{(1) Bounded variance.} First note that the variance of the entropy term is zero, hence
\begin{align*}
    \E\left[ \norm{\hat{H}(z) - H(z)}^2 \right]
    &= \E\left[ \norm{\hat{F}(z) - F(z)}^2 \right] + \E\left[ \norm{\hat{\boldsymbol{s}}(\pi_x, \Lambda) - \boldsymbol{s}(\pi_x, \Lambda)}^2 \right] \\
    &\leq \E\left[ \norm{\hat{F}(z)}^2 \right] + \E\left[ \norm{\hat{\boldsymbol{s}}(\pi_x, \Lambda)}^2 \right] ,.
\end{align*}
For the second term, we easily obtain
\begin{align*}
    \E\left[ \norm{\hat{\boldsymbol{s}}(\pi_x, \Lambda)}^2 \right] \leq \sum_i \E\left[ \hat{s}(\pi_x, \lambda_i)^2 \right] \leq 4 N R_*^2
\end{align*}
for both $s = -J$ and $s = \text{Reg}$. For the first term, we invoke Lipschitzness and $\zeta$-greediness of the policy $\pi$. For any trajectory $\tau$ and any level $\lambda$, we have $\abs{\sum_t \Psi^t(\tau, \lambda)} \leq 2T R_*$ for any choice of estimator $\Psi^t \in \{ R^t, Q_{\pi_x}^t, A_{\pi_x}^t \}$, which combined with
\begin{align*}
    \norm{\nabla_x \log \pi_x(\tau^t, \lambda)} &= \frac{\norm{\nabla_x \pi_x(\tau^t, \lambda)}}{\pi_x(\tau^t, \lambda)} \leq \frac{L}{\zeta} \,,
\end{align*}
implies that
\begin{align*}
    \E\left[ \norm{\hat{F}(z)}^2 \right] \leq \norm{\sum_t \nabla_x \log \pi_x(\tau^t, \lambda) \Psi^t(\tau, \lambda)}^2 &\leq \frac{4 T^2 R_*^2 L^2}{\zeta^2}
\end{align*}
and hence
\begin{align*}
    \E\left[ \norm{\hat{H}(z) - H(z)}^2 \right] &\leq 4 N R_*^2 + \frac{4 T^2 R_*^2 L^2}{\zeta^2} = \sigma^2_1
\end{align*}
as required. The variance for arbitrary $M$ follows immediately as $\sigma_M^2 = \sigma_1^2 / M$.

\textbf{(2) Strong concavity of $f$ in $y$.} Trivial, since $\nabla_y^2 f(x, y) = \text{diag}(-\alpha/y) \preceq -\alpha I$.

\textbf{(3) Lipschitzness of $f$.} First note that
\begin{align*}
    \norm{ \nabla_y \cH }^2 = \sum_i (1+\log y_i)^2 \leq N (1+\log \xi)^2 \,.
\end{align*}
We combine this with Jensen's inequality and part \textbf{(1)} of the proof above to obtain $L$-Lipschitzness:
\begin{align*}
    \norm{ \nabla f }^2 &= \norm{ \E\big[ \hat{H} \big] }^2 \leq \E\left[ \norm{ \hat{H} }^2 \right] \\
    &= \E\left[ \norm{ \hat{F} }^2 \right] + \E\left[ \norm{ \hat{\boldsymbol{s}}  + \alpha \nabla_y \cH}^2 \right] \\
    &= \E\left[ \norm{ \hat{F} }^2 \right] + \E\left[ \norm{ \hat{\boldsymbol{s}}}^2 \right] + \alpha^2 \norm{ \nabla_y \cH }^2 + 2\alpha \E\left[ \norm{ \hat{\boldsymbol{s}} } \right] \norm{\nabla_y \cH } \\
    &\leq \sigma_1^2 + \alpha^2 N (1+\log \xi)^2 + 4\alpha N R_*(1+\log \xi) \eqqcolon L \,.
\end{align*}

\textbf{(4) Smoothness of $f$.} The policy gradient Hessian is given by \citep{shen2019hessian}
$$
\nabla_x^2 J(\pi_x, \lambda) = \E_\tau \bigg[ \sum_t R^t(\tau, \lambda) \Big( \nabla_x \log \pi_x(\tau^t) \nabla_x \log p(\tau \mid \pi_x)^T + \nabla_x^2 \log \pi_x(\tau^t) \Big) \bigg] \,,
$$
where $p(\tau \mid \pi_x) = \rho(s_0) \prod_t P(s_{t+1} \mid s_t, a_t) \pi_x(\tau^t)$ for initial and transition distributions $\rho$ and $P$ (omitting $\lambda$ for convenience). For the first term, writing $P(\tau) = \prod_t P(s_{t+1} \mid s_t, a_t)$, we have
\begin{align*}
    \nabla_x p(\tau \mid \pi_x) = \rho(s_0) P(\tau) \sum_t \nabla_x \pi_x(\tau^t) \prod_{s \neq t} \pi_x(\tau^t)
\end{align*}
which implies
\begin{align*}
    \norm{\nabla_x p(\tau \mid \pi_x)} \leq TL \,,
\end{align*}
hence the first term is bounded for each $t$:
\begin{align*}
    \norm{\nabla_x \log \pi_x(\tau^t) \nabla_x \log p(\tau \mid \pi_x)^T} \leq \frac{TL^2}{\zeta} \,.
\end{align*}
For the second term, we use the $K$-smoothness of $\pi$ to obtain:
\begin{align*}
    \norm{\nabla_x^2 \log \pi_x(\tau^t)} &= \norm{\frac{\nabla_x^2 \pi_x(\tau^t)}{\pi_x(\tau^t)} - \frac{\nabla_x \pi_x(\tau^t) \nabla_x \pi_x(\tau^t)^T}{\pi_x(\tau^t)^2}} \leq \frac{K}{\zeta} + \frac{L^2}{\zeta^2} \,.
\end{align*}
Putting everything together, recalling that $\abs{R^t(\tau, \lambda)} \leq R_*$ for all $t$, we obtain
\begin{align*}
    \norm{\nabla_x^2 f(z)} &= \norm{ y^T \nabla_x^2 J(\pi_x, \Lambda) } \leq \max_\lambda \norm{ \nabla_x^2 J(\pi_x, \lambda) } \leq \frac{T R_*}{\zeta} \left( TL^2 + K + \frac{L^2}{\zeta} \right) \,.
\end{align*}
Now notice that $\norm{\nabla_{y}^2 f(z)} = \norm{\text{diag}(-\alpha/y)} \leq \alpha/\xi$ since $y_i \leq \xi$ for all $i$. Moreover,
$$ \norm{\nabla_{xy}^2 f(z)} = \norm{\nabla_x J(\pi_x, \Lambda)} \leq \frac{2TR_*L}{\zeta} $$
by the same argument as part \textbf{(1)} of the proof, so we conclude
$$ \norm{\nabla^2 f(z)} \leq \frac{TR_*}{\zeta} \left( TL^2 + K + \frac{L^2}{\zeta} + 2TL \right) + \frac{\alpha}{\xi} = \ell $$
as required.
\end{proof}

\subsubsection*{Acknowledgments}
The authors would like to thank Jingming Yan, Ioannis Panageas, and JB Lanier for their helpful feedback on the paper.  MJ and AG are funded by the EPSRC Centre
for Doctoral Training in Autonomous Intelligent Machines and Systems. MJ is also funded by Amazon Web Services. MB is funded by the Rhodes Trust. AR is partially funded by the EPSRC Programme Grant ``From Sensing to Collaboration'' (EP/V000748/1).
\bibliography{main}
\bibliographystyle{rlj}

\beginSupplementaryMaterials

\section{Additional Experimental Details}\label{sec:details}

\subsection{Practical Algorithm}
In practice, we often (although not always) find it helpful to use a dynamic buffer, as per Section \ref{sec:dynamic}. Moreover, we define the subroutine \textsc{train\_rl} to refer to any form of training from the reinforcement learning literature, however in practice we use mini-batch PPO \citep{schulman2017proximalpolicyoptimizationalgorithms}. Moreover, we use TiAda-Adam \citep{li2022tiadatimescaleadaptivealgorithm} as an adaptive optimiser for our optimisation setting, but for all other experiments we use Adam \citep{kingma2017adammethodstochasticoptimization}. We write details for the dynamic buffer in \textcolor{brightmaroon}{maroon}.

\begin{algorithm}[h]
    \caption{Practical Nonconvex-concave Optimisation for UED \textcolor{brightmaroon}{(Dynamic Buffer)}}

    \begin{algorithmic}\label{alg:practice}
        \REQUIRE{Initial policy $x^{0}$, distribution $y^0 = \frac{1}{|\Lambda|}\mathbf{1}$, stepsizes $\eta_x, \eta_y$}, \textcolor{brightmaroon}{initial} level set $\Lambda^{0}$.
        \FOR{$t = 0, 1, \ldots$} 
            \STATE Sample batch of training levels $\boldsymbol{\lambda} \sim \Lambda^{t}(y^{t})$
            \STATE Construct score vector $\boldsymbol{s}(\pi_x, \Lambda)$
            \textcolor{brightmaroon}{\STATE Sample new levels $\Lambda' \sim \mathcal{L}$ 
            \STATE Construct alternate score vector $\mathbf{s}' = s(\pi_x, \Lambda')$
            \STATE $\Lambda^{t+1} = $ top $|\Lambda|$ elements from $\Lambda^{t} \cap \Lambda'$
            \STATE Construct merged score vector $\Tilde{\mathbf{s}} = \boldsymbol{s}(\pi_x, \Lambda^{t})$}
            \STATE $x^{t+1} = $ \textsc{train\_rl}($x^{t}$, $\boldsymbol{\lambda}$, $\eta_x$)
            \STATE $y^{t+1} = \mathcal{P}_{\mathcal{Y}} \bigg(y^{t} + \eta_y \cdot \hat{G}(x^{t}, y^{t}; \mathbf{s},
            \textcolor{brightmaroon}{\Tilde{\mathbf{s}}}) \bigg)$ with $\hat{G}$ defined in Equation \eqref{eq:H}
           
        \ENDFOR\\
        \STATE \textbf{return} Best-iterate policy parameters $x^*$
    \end{algorithmic}
\end{algorithm}

\subsection{Additional Craftax Details}\label{sec:craftax_dets}
We maintain the training regime of \citet{matthews2024craftaxlightningfastbenchmarkopenended} by using ``inner'' and ``outer'' rollouts, where we update after multiple parallelised sub-sequences within an episode. Moreover, due to such a level space where the levels and test set are so similar, we found that it more effective to \textit{anneal} our entropy regularisation coefficient $\alpha$ by the rule $\alpha^{t} = \frac{\alpha}{\sqrt[3]{t+1}}$, thus resulting in a more diverse set of training levels at the start of training. 

\subsection{Comparison Between Theory and Practice}\label{sec:comparison}
We give the following side by side comparison between our practical and theoretical method:
\begin{table}[H]
    \centering
    \caption{Qualitative comparison between theoretical NCC and practical NCC.}
    \begin{tabular}{@{}lll@{}}
        \toprule
        {} & Theory & Practice \\
        \midrule
        $x$ Gradient Estimator & Equation \ref{eq:G} (REINFORCE) & Any \\
        \# minibatches (epochs) & 1 (1) & Any (Any) \\
        Dynamic Buffer & Not Allowed & Allowed \\
        Optimiser & SGD & Any \\
        Score Function & Only zero-sum (i.e. Regret and $-J$) & Any  \\ 
        Activation Function & Smooth & Any \\
        \bottomrule
    \end{tabular}
    \label{tab:comparison}
\end{table}

\begin{figure}[htbp]
    \centering
    \begin{subfigure}[b]{0.475\textwidth}
        \centering
        \includegraphics[width=\linewidth]{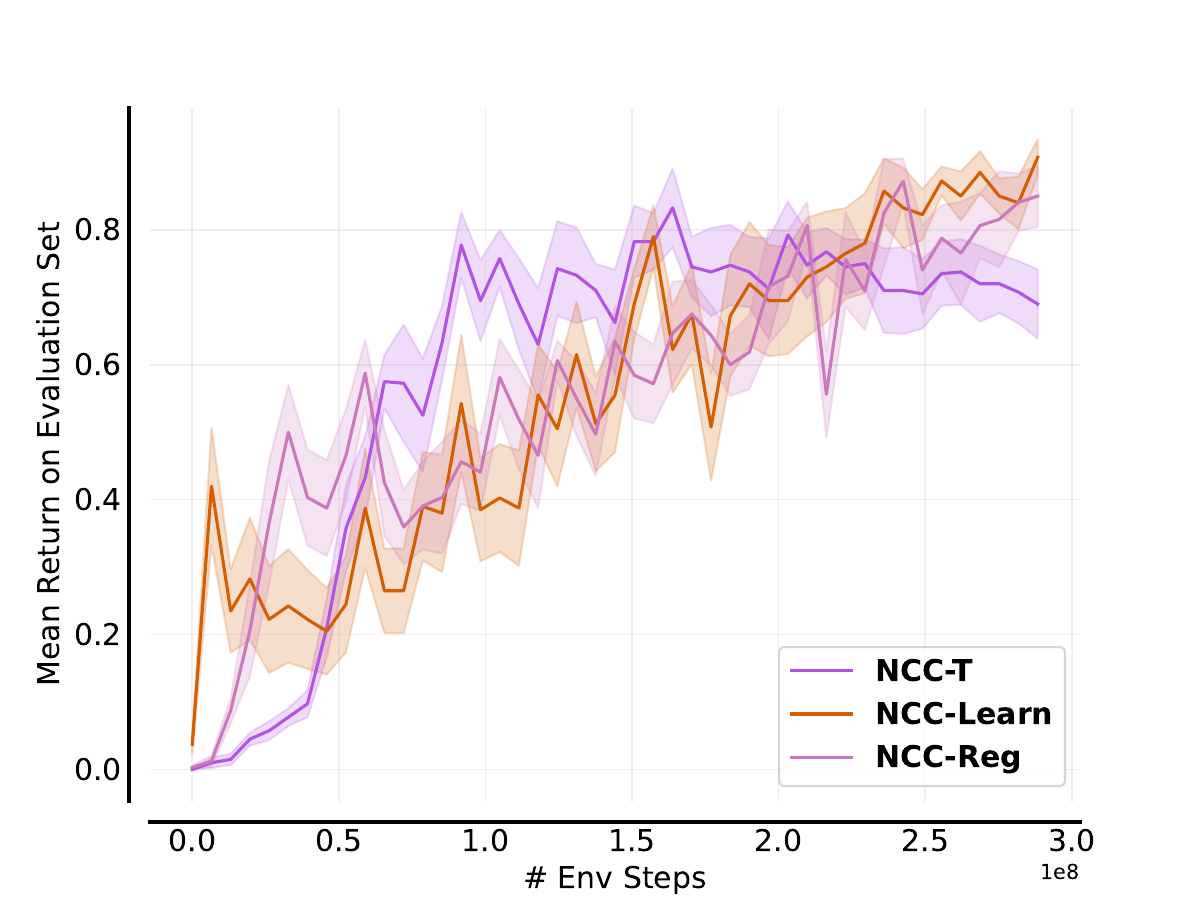}  
        \caption{Solve Rates}
    \end{subfigure}
    \hfill
    \begin{subfigure}[b]{0.475\textwidth}
        \centering
        \includegraphics[width=0.875\linewidth]{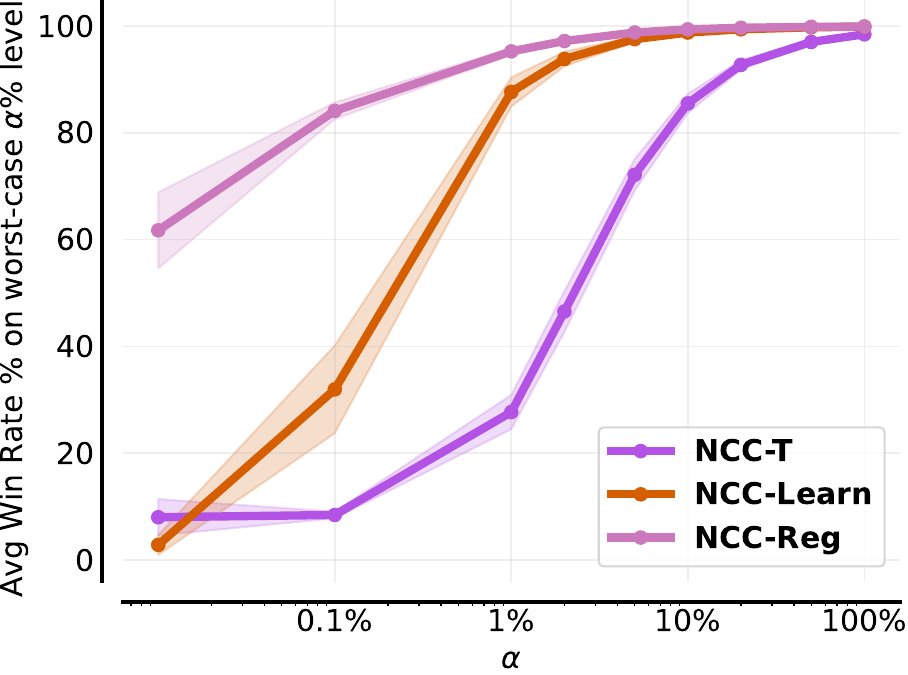}
        \caption{$\alpha$-CVar Evaluation}
    \end{subfigure}
    \caption{Empirical comparison between the different instantiations of NCC.}\label{fig:ncc_comparison}
\end{figure}

Moreover, in Figure \ref{fig:ncc_comparison} we experimentally compare the theoretical version of our method (``NCC-T'') with the other methods on Minigrid, which is our only testbed with a regret oracle. 

\subsection{Generalised Learnability Score Function}
\label{sec:gen_learn_supp}

In Figure~\ref{fig:gen_learn_ana} we repeat the analysis of UED score functions conducted by~\citet{rutherford2024regretsinvestigatingimprovingregret}. To give us a success rate metric, we conduct this analysis in Minigrid using a policy trained for 1100 update steps with SFL ($1/4$ of a usual training run). We randomly sample 5000 levels and rollout the policy for 2000 timesteps on each. The trend illustrated by the quadratic demonstrates the generalised learnability score function's ability to identify levels of intermediate difficulty.

\begin{figure}[h]
    \centering
    \begin{tabular}{cc}
        \includegraphics[width=0.475\linewidth]{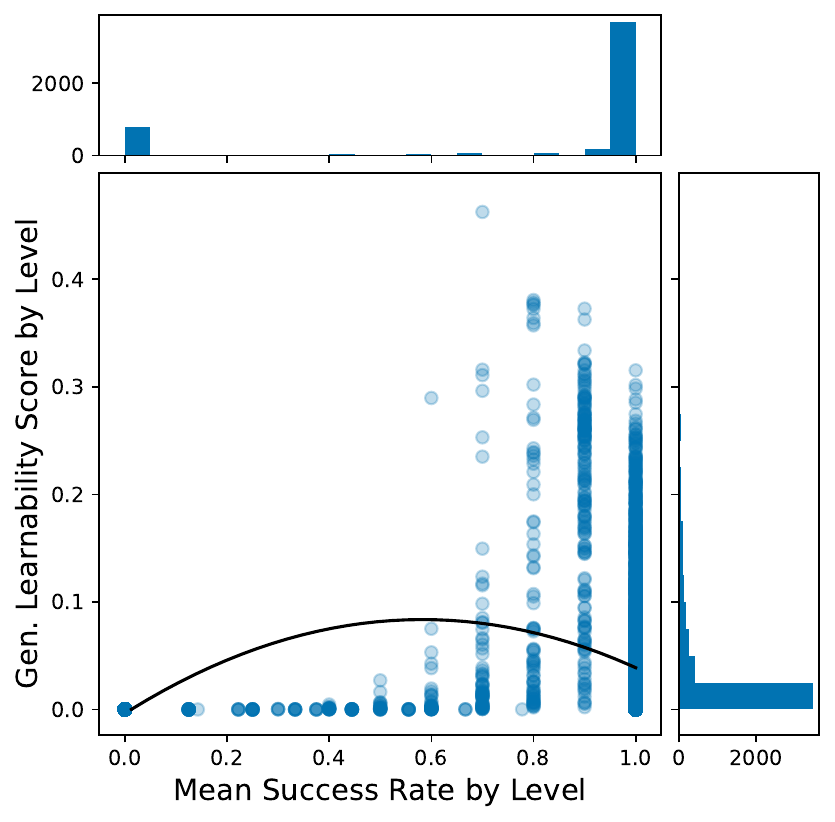}  &\includegraphics[width=0.475\linewidth]{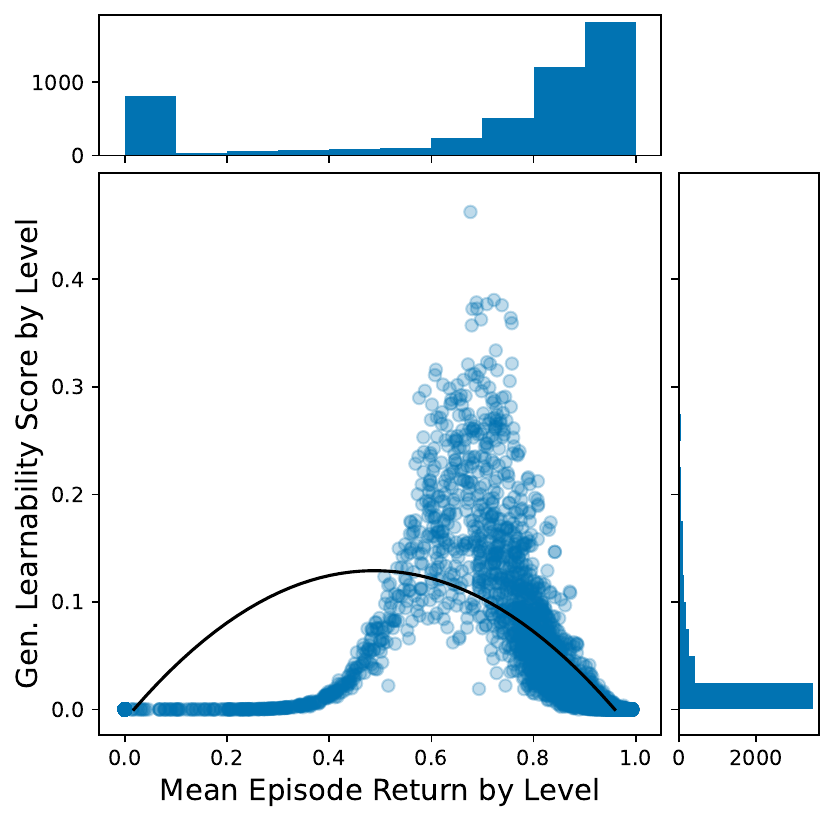} \\
    \end{tabular}
        
    \caption{Analysis of Generalised Learnability Score function on Minigrid. The black lines represent a quadratic fit to the scatter data.}
    \label{fig:gen_learn_ana}
\end{figure}

\subsection{Robustness Scores}\label{sec:cvar}
\begin{figure}[H]
    \centering
    \begin{subfigure}[b]{0.475\textwidth}
        \centering
        \includegraphics[width=\linewidth]{{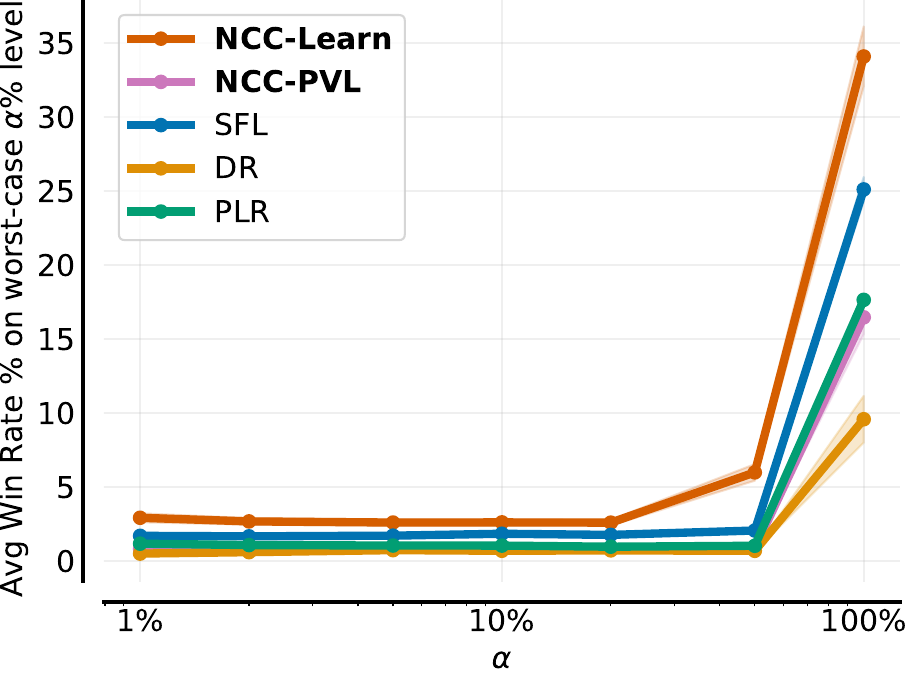}}
        \caption{XLand-Minigrid}
    \end{subfigure}
    \hfill
    \begin{subfigure}[b]{0.475\textwidth}
        \centering
        \includegraphics[width=\linewidth]{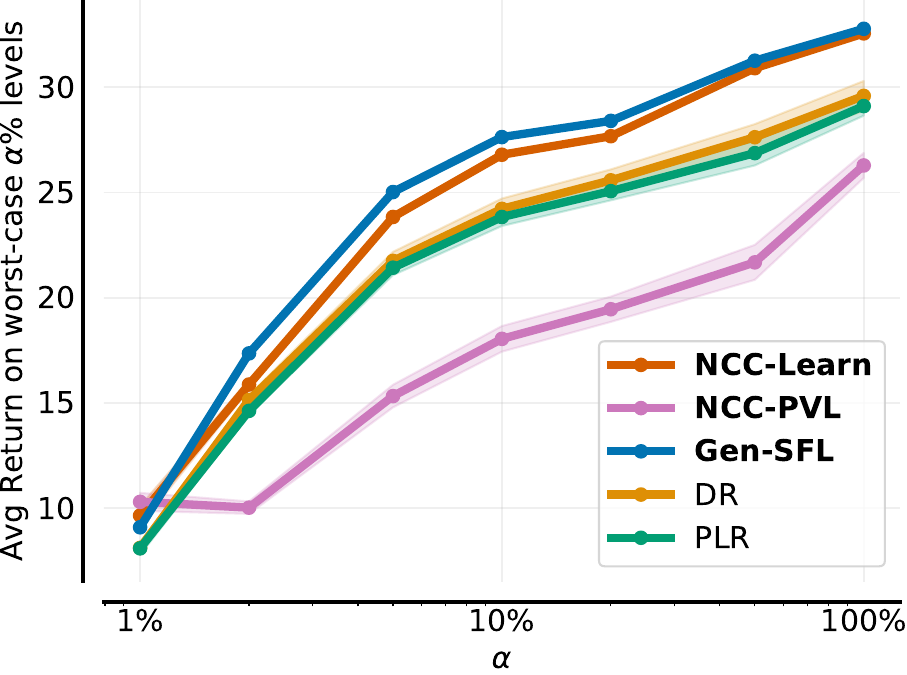} 
        \caption{Craftax}
    \end{subfigure}
    \caption{Omitted $\alpha$-CVaR Robustness Evaluation curves.}
\end{figure}

\subsection{Compute Time}\label{time}

Table \ref{tab:compute_time} reports the compute time for all experimental evaluations. Each Minigrid seed was run on 1 Nvidia A40 using a server that has 8 Nvidia A40's and two AMD EPYC 7513 32-Core Processor (64 cores in total). Meanwhile, for XLand and Craftax, each individual seed was run on 1 Nvidia L40s using a server that has 8 NVIDIA L40s’, two AMD EPYC 9554 processors (128 cores in total). However, NCC experiments for Minigrid were run on the L40s server.

\begin{table}[h]
    \centering
    \caption{Mean and standard deviation of time take for experimental evaluations. Each evaluation consisted of 10 independent seeds.}
    \begin{tabular}{@{}lrrr@{}}
        \toprule
         Method & Minigrid & XLand & Craftax \\
         \midrule
         NCC Learn & 1:02:41 (0:00:26) & 3:31:57 (0:01:06) & 5:35:06 (0:00:59) \\
         NCC Regret & 1:01:59 (0:00:11) & - & - \\
         NCC PVL & - & 2:52:31 (0:00:53) & 4:13:16 (0:00:39) \\
         SFL & 0:28:19 (0:00:03) & 9:16:31 (0:01:17) & 4:29:17 (0:00:31) \\
         PLR & 0:45:16 (0:00:10) & 2:47:53 (0:00:47) & 3:24:46 (0:00:24) \\
         DR & 0:43:28 (0:00:16) & 2:42:27 (0:00:43) & 3:28:06 (0:00:08) \\
         \bottomrule
    \end{tabular}
    \label{tab:compute_time}
\end{table}
\clearpage
\subsection{Hyperparameters}\label{sec:hypers}

\begin{table}[h]
    \centering
    \caption{NCC (Reg) Hyperparameters}
    \label{fig:hypers_ncc}
    \begin{tabular}{@{}lrrr@{}}
        \toprule
        Hyperparameter & Minigrid  & XLand & Craftax \\
        \midrule
        $\eta_x$ & $0.001$  & $0.0001$ & $0.0001$ \\
        $\eta_y$ & $0.1$ $(0.005)$ & $0.01$& $0.01$ \\
        $\alpha$ & $0.05$ $(0.032)$ & 0 & $0.05$ \\
        $|\Lambda|$ & 4000 & 4000 & 4000 \\
        $|\Lambda'|$ & 256 & 8192 & 0 \\    $|\boldsymbol{\lambda}|$ & 256 & 8192 & 1024 \\
        $\gamma$ & 0.995 & 0.99 & 0.995 \\
        GAE $\lambda$ &0.98 & 0.95 & 0.95 \\
        \verb|clip_eps| & 0.2 & 0.2 & 0.2 \\
        \verb|critic_coeff| & 0.5 & 0.5 & 0.5 \\
        \verb|entropy_coeff| & 0.001 & 0.01 & 0.01 \\
        \verb|num_epochs| & 1 & 1 & 4 \\
        \verb|max_grad_norm| & 0.25 & 0.5 & 1.0 \\
        \verb|num_minibatches| & 1 & 16 & 2 \\
        \verb|num_parallel_envs| & 256 & 8192 & 1024 \\
        \bottomrule
    \end{tabular}
\end{table}

\begin{table}[h]
    \centering
    \caption{{PLR (DR)} Hyperparameters}
    \label{fig:hypers_plr}
    \begin{tabular}{@{}lrrr@{}}
        \toprule
        Hyperparameter & Minigrid  & XLand & Craftax \\
        \midrule
        $\eta_x$ & $0.00025$  & $0.0001$ & $0.0002$ \\
        $|\Lambda|$ & 4000 & 4000 & 4000 \\
        $\gamma$ & 0.995 & 0.99 & 0.99 \\
        GAE $\lambda$ &0.98 & 0.95 & 0.9 \\
        \verb|clip_eps| & 0.2 & 0.2 & 0.2 \\
        \verb|critic_coeff| & 0.5 & 0.5 & 0.5 \\
        \verb|entropy_coeff| & 0 & 0.01 & 0.01 \\
        \verb|num_epochs| & 4 & 1 & 5 \\
        \verb|max_grad_norm| & 0.5 & 0.5 & 1.0 \\
        \verb|num_minibatches| & 4 & 16 & 2 \\
        \verb|num_parallel_envs| & 256 & 8192 & 1024 \\
        \verb|replay_prob| & 0.5 (0) & 0.95 (0) & 0.5 (0) \\
        \verb|staleness_coeff| & 0.3 & 0.3 & 0.3 \\
        \verb|temperature| & 1 & 1 & 1 \\
        \bottomrule
    \end{tabular}
\end{table}

\begin{table}[h]
    \centering
    \caption{SFL Hyperparameters}
    \label{fig:hypers_sfl}
    \begin{tabular}{@{}lrrr@{}}
        \toprule
        Hyperparameter & Minigrid  & XLand & Craftax \\
        \midrule
        $\eta_x$ & $0.00025$  & $0.001$ & $0.0001$ \\
        $|\Lambda|$ & 1000 & 8192 & 4000 \\
        $\gamma$ & 0.99 & 0.99 & 0.995 \\
        GAE $\lambda$ &0.95 & 0.95 & 0.95 \\
        \verb|clip_eps| & 0.04 & 0.2 & 0.2 \\
        \verb|critic_coeff| & 0.5 & 0.5 & 0.5 \\
        \verb|entropy_coeff| & 0 & 0.01 & 0.01 \\
        \verb|num_epochs| & 4 & 1 & 4 \\
        \verb|max_grad_norm| & 0.5 & 0.5 & 1.0 \\
        \verb|num_minibatches| & 4 & 16 & 2 \\
        \verb|num_parallel_envs| & 256 & 8192 & 1024 \\
        \verb|batch_size| & 4000 & 40000 & 4000\\
        \verb|num_batches| & 5 & 1 & 5 \\
        \bottomrule
    \end{tabular}
\end{table}

\clearpage
\section{Difficulty of Levels}\label{sec:difficulty}
To show how our method evolves over time, we compare minigrid levels at halfway and final timesteps in training. Firstly, we plot levels from DR in Figure \ref{fig:dr_grids}. Levels from DR are not well selected, as there are unsolvable levels, as well as \textit{trivial} levels at the end of training. Secondly, as is explainable by them both selecting for learnability, NCC with learnability (Figure \ref{fig:ncc_grids}) and SFL (Figure \ref{fig:sfl_grids}) both have what appear to be difficult (but not impossible) levels halfway and at the end of training, although we do note that NCC appears to weigh some levels with shorter optimal paths at the end of training in comparison to SFL (particularly the left and middle levels of NCC). This may be to retain diversity in the difficulty of the batch of sampled levels, to prevent overfitting to a certain class of problems. However, our analysis is a hypothesis, as our approach is learned, meaning it is more black-box (i.e. uninterpretable).
\begin{figure}[h]
    \centering
    
        \begin{tabular}{ccc}
            \includegraphics[width=0.30\linewidth]{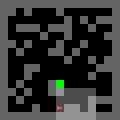} & \includegraphics[width=0.30\linewidth]{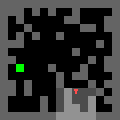}
             & \includegraphics[width=0.30\linewidth]{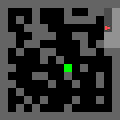} \\
             \includegraphics[width=0.30\linewidth]{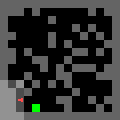} & \includegraphics[width=0.30\linewidth]{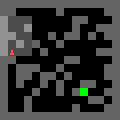}
             & \includegraphics[width=0.30\linewidth]{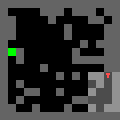} 
        \end{tabular}
        \caption{DR: Sampled levels at halfway through training (top row) and the end of training (bottom row)}
    \label{fig:dr_grids}
\end{figure}

\begin{figure}[h]
    \centering
    
        \begin{tabular}{ccc}
            \includegraphics[width=0.30\linewidth]{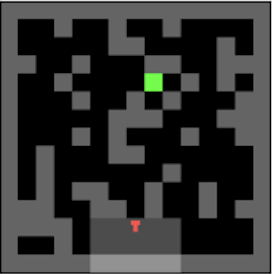} & \includegraphics[width=0.30\linewidth]{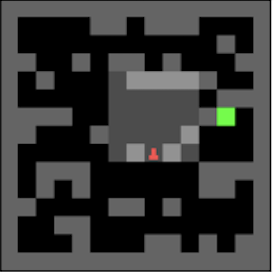}
             & \includegraphics[width=0.30\linewidth]{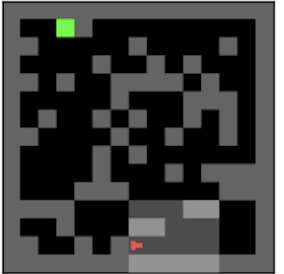} \\
             \includegraphics[width=0.30\linewidth]{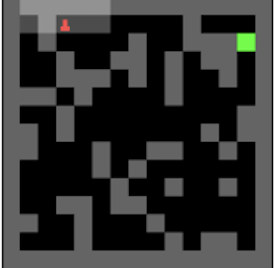} & \includegraphics[width=0.30\linewidth]{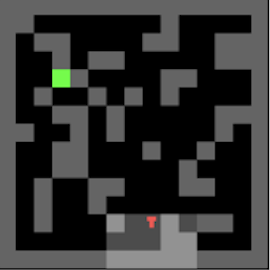}
             & \includegraphics[width=0.30\linewidth]{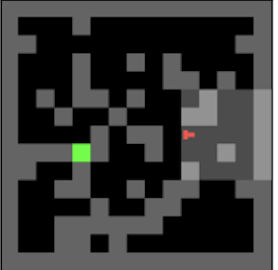} 
        \end{tabular}
        \caption{SFL: Highest learnability scoring levels at halfway through training (top row) and the end of training (bottom row)}
    \label{fig:sfl_grids}
\end{figure}

\begin{figure}[h]
    \centering
    
        \begin{tabular}{ccc}
            \includegraphics[width=0.30\linewidth]{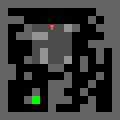} & \includegraphics[width=0.30\linewidth]{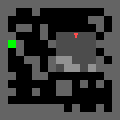}
             & \includegraphics[width=0.30\linewidth]{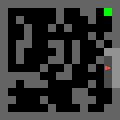} \\
             \includegraphics[width=0.30\linewidth]{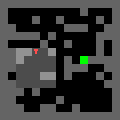} & \includegraphics[width=0.30\linewidth]{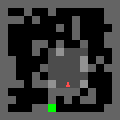}
             & \includegraphics[width=0.30\linewidth]{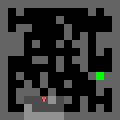} 
        \end{tabular}
        \caption{NCC-Learn: Highest weighted levels at halfway through training (top row) and the end of training (bottom row)}
    \label{fig:ncc_grids}
\end{figure}

\end{document}